\documentclass[conference]{IEEEtran}
\usepackage{times}

\usepackage[numbers]{natbib}
\usepackage{multicol}
\usepackage[bookmarks=true]{hyperref}

\usepackage{ifthen}
\usepackage[usenames,dvipsnames]{color}
\usepackage{bm}
\usepackage[pdftex]{graphicx}
\usepackage{subcaption}
\usepackage{graphics}
\usepackage{amsmath}

\usepackage{algorithm}
\usepackage[noend]{algpseudocode}
\usepackage{sidecap}
\usepackage{amsthm}
\usepackage{comment}

\newcommand{\eref}[1]{(\ref{#1})}
\newcommand{\secref}[1]{Section~\ref{#1}}
\newcommand{\tabref}[1]{Table~\ref{#1}}

\newcommand{\figref}[1]{Fig.~\ref{#1}}

\newtheorem{theorem}{Theorem}

\newcommand*{\Cdot}{\raisebox{-0.25ex}{\scalebox{1.75}{$\cdot$}}}
\newcommand{\myparagraph}[1]{\vspace{0.05in}\noindent\textbf{#1}}

\newboolean{draft-mode}
\setboolean{draft-mode}{true}
\newcommand{\sidenote}[1]{\ifthenelse{\boolean{draft-mode}}{\marginpar{\tiny\raggedright\textsf{\hspace{0pt}#1}}}{}}
\DeclareRobustCommand{\arnote}[1]{\ifthenelse{\boolean{draft-mode}}{\textcolor{blue}{\textbf{AR: #1}}}{}}
\DeclareRobustCommand{\ncdnote}[1]{\ifthenelse{\boolean{draft-mode}}{\textcolor{magenta}{\textbf{NCD: #1}}}{}}
\DeclareRobustCommand{\rhnote}[1]{\ifthenelse{\boolean{draft-mode}}{\textcolor{red}{\textbf{RH: #1}}}{}}

\begin{document}

\title{In-Hand Manipulation via Motion Cones}


\author{{Nikhil Chavan-Dafle, Rachel Holladay, and Alberto Rodriguez}\\
{Massachusetts Institute of Technology}\\
\{nikhilcd, rhollada, albertor\}@mit.edu}



\maketitle

\begin{abstract}
In this paper, we present the mechanics and algorithms to compute the set of feasible motions of an object pushed in a plane. 
This set is known as the \emph{motion cone} and was previously described for non-prehensile manipulation tasks in the horizontal plane.
We generalize its geometric construction to a broader set of planar tasks, where external forces such as gravity influence the dynamics of pushing, and prehensile tasks, where there are complex interactions between the gripper, object, and pusher.
We show that the motion cone is defined by a set of low-curvature surfaces and provide a polyhedral cone approximation to it.
We verify its validity with 2000 pushing experiments recorded with motion tracking system.

Motion cones abstract the algebra involved in simulating frictional pushing by providing bounds on the set of feasible motions and by characterizing which pushes will stick or slip.
We demonstrate their use for the dynamic propagation step in a sampling-based planning algorithm for in-hand manipulation. %
The planner generates trajectories that involve sequences of continuous pushes with 5-1000x speed improvements to equivalent algorithms.
\href{https://youtu.be/tVDO8QMuYhc}{\textcolor{RoyalBlue}{Video Summary -- youtu.be/tVDO8QMuYhc}}

%
%
%
%
%

\end{abstract}

\IEEEpeerreviewmaketitle

\section{Introduction}
\label{sec:intro}

A \emph{motion cone} is the set of feasible motions that a rigid body can follow under the action of a frictional push.
We can think of it as a geometric representation of the underactuation inherent to frictional contacts.
%
Since contacts can only push, and since friction is limited, a contact can move an object only along a limited set of rays.
The concept was introduced by~\citet{mason86} for point contacts in the context of a planar horizontal pushing task.

Motion cones abstract the algebra involved in simulating frictional contact dynamics.
A contact force on the inside (or boundary) of the friction cone produces sticking (or slipping) behavior, and leads to motion rays on the inside (or boundary) of the motion cone. 
\citet{lynch96} generalized the construction of motion cones to line contacts in a horizontal plane. They used them to plan stable pushing trajectories without the complexities of standard complementarity formulations of contact dynamics~\citep{Stewart96, Posa2014, ChavanDafle2017}.
Motion cones have since been the basis of several efficient planning and control strategies for planar manipulations on a horizontal support surface~\citep{Erdmann98, Dogar10, Hogan16, Jiaji17b, Hogan17}.

This paper studies the construction of motion cones in a more general set of planar tasks. In particular, we highlight the case of prehensile manipulation in the vertical plane. In general planar tasks, external forces other than the pusher force (e.g., gravity) can alter the dynamics of contact interactions between the pusher, object, and gripper/support-plane. Motion cones efficiently capture the intricate mechanics of these tasks for simulation, planning, and control.

We present three main contributions:
\begin{itemize}
    \item \textbf{Mechanics} of motion cones for planar tasks in the gravity plane. We show that the motion cone is defined by a set of low-curvature surfaces, intersecting at a point and pairwise in lines. We propose a polyhedral approximation to the motion cone for efficient computation.
    \item \textbf{Experimental validation} of the stick/slip condition of motion cones in a prehensile pushing task instrumented with a Vicon motion tracker.
    \item \textbf{Application} of motion cones in a sampling-based planning framework for in-hand manipulation using prehensile pushes (see \figref{fig:cone_planning}). We show this yields significant speed improvements with respect to  our prior work~\citep{ChavanDafle2017, ChavanDafle2018a}. 
\end{itemize}

\begin{figure}
\centering
\includegraphics[scale=0.9]{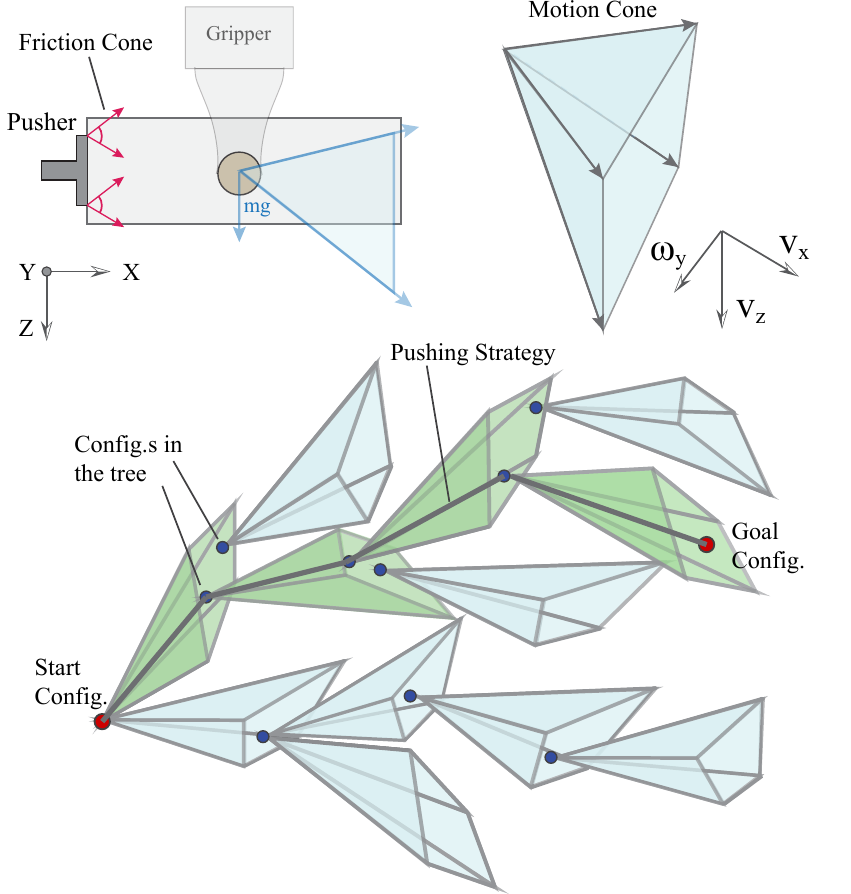}
\caption{(top) Example \emph{friction cone} and \emph{motion cone} of an object moving in the vertical plane. The pusher can move the object along any direction $[V_x, V_z, \omega_y]$ inside the motion cone. 
(bottom) A plan via motion cones. Motion cones capture local reachability. A path in the tree of motion cones generates a pushing strategy to move an object.}
\label{fig:cone_planning}
\end{figure} 

\begin{figure*}[t]
\centering
\includegraphics[scale=0.205]{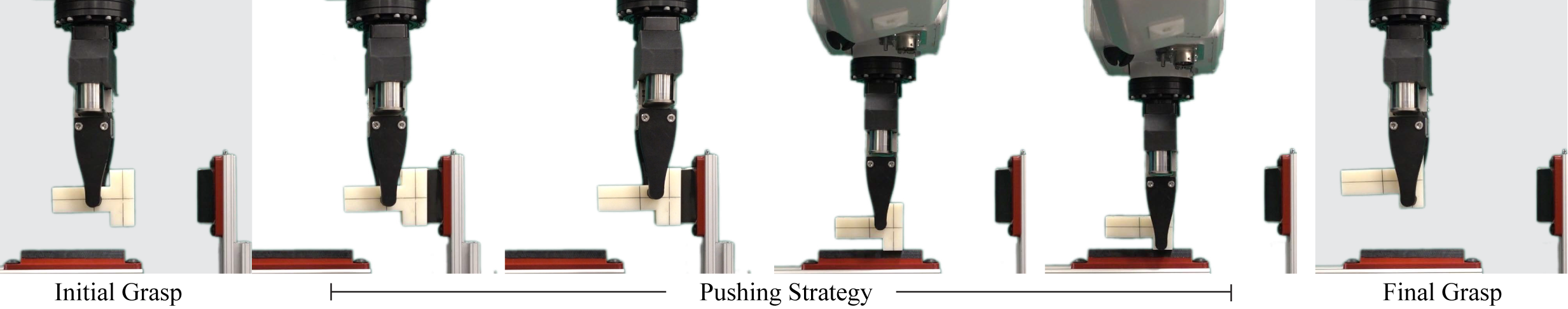}
\caption{Manipulating a T-shaped object in a parallel-jaw grasp by pushing it against features in the environment. The manipulation is shown from a side view.}
\label{fig:tpush_exp}
\vspace{-2mm}
\end{figure*}

%
%
\figref{fig:tpush_exp} shows an example of a  pushing trajectory planned to change the grasp on a T-shaped object. The resulting trajectory is a sequence of continuous stable pushes, each where the object sticks to a particular external pusher. The proposed planning algorithm obtains these trajectories consistently in less than a second.

Section~\ref{sec:pushing_mechanics} reviews the mechanics of pushing in an arbitrary plane, while \secref{sec:cone_mechanics} describes the process to construct the corresponding motion cones. In \secref{sec:motioncone_computation} we discuss the computational aspects in the calculation and approximation of motion cones for pushing an object gripped with a finite force. In \secref{sec:planning} we demonstrate a sampling-based planning approach that exploits motion cones to speed up the computation of local contact dynamics for prehensile pushing.

The generalization of motion cones to interactions with gravity opens a door for efficient and robust planning of in-hand manipulations that respect and exploit the basic principles of frictional rigid-body contact interactions:
Newton's second law, Coulomb's friction law, the principle of Maximal Dissipation, and the rigidity of rigid-bodies.

\section{Related Work}
\label{sec:related}

Planning and control through contact is a central topic in robotic manipulation research. 
Rigid-body contact is modeled as a series of constraints on the possible motions and forces at contact. Simplifying and exploiting these constraints is a pivotal theme in the non-prehensile manipulation literature.
\citet{GoyalPhD89} introduced the concept of a \textit{limit surface}, a compact mapping from the friction wrench between slider and ground and the sliding twist at contact. \citet{mason86} studied the mechanics of pushing and proposed the concept of the \textit{motion cone}. These two fundamental geometric constructions provide direct force-motion mappings for contact interactions and have facilitated efficient planning and control techniques in non-prehensile manipulation~\citep{lynch92, lynch96, Dogar10, Dogar2011, Hogan16, Jiaji17b}.

Recent work on trajectory optimization and manipulation planning shows that it is possible to reason about different contact modes and plan trajectories through contact in a standard rigid body dynamics framework based on complementarity constraints. These methods often have to compromise on the computational efficiency~\citep{Posa2014,ChavanDafle2017} or the realism of contact dynamics~\citep{Kumar14,todorov2012,kolbert16}.

In contrast, some more recent work has focused on particular types of manipulation primitives and exploits assumptions in the problem formulation to develop fast planning and robust control strategies.
\citet{lynch15} demonstrates dynamic in-hand manipulation planning in a parallel-jaw grasp. With a pre-defined contact mode sequence at the fingers and a limit surface approximation for the force-motion interaction at the fingers, they derive a control law that can move the object to the goal pose in the grasp.
Similar approaches are explored for planning and controlling in-hand manipulations by actively using gravity~\citep{kragic_pivoting} or dynamic motions~\cite{holladay15,sintov16,Yifan17}. 

\citet{Sundaralingam17} propose a trajectory-optimization based, purely-kinematic approach for in-hand manipulation with a multi-finger gripper. They assume that the fingers on the object do not slip and impose soft constraints that in practice minimize the slip at the fingers. By assuming all finger contacts to be sticking, they bypass the need for modelling the dynamics of the manipulations and obtain kinematic plans quickly. 

In our recent work~\citep{ChavanDafle2018a}, we present a fast sampling-based planning framework for in-hand manipulations with prehensile pushes, where the pusher contact is guaranteed to stick to the object using the mechanics of the manipulations. 
The plans are discrete sequences of continuous pushes that respect friction, contact, and rigid-body constraints.


These promising results on in-hand manipulation research, which limit to certain types of manipulation primitives for efficient planing and control, motivate us to generalize the concept of motion cones to more general pushing tasks.
Motion cones, as a set of direct constraints on the object motion or control inputs, can naturally fit well in a trajectory optimization framework. 
As we illustrate in this paper, for sampling-based methods, the motion cone can be used to guide sampling, and for fast dynamics propagation. 
\section{Mechanics of Pushing in a Plane}
\label{sec:pushing_mechanics}
\begin{figure*}[t]
\centering
\includegraphics[scale=0.95]{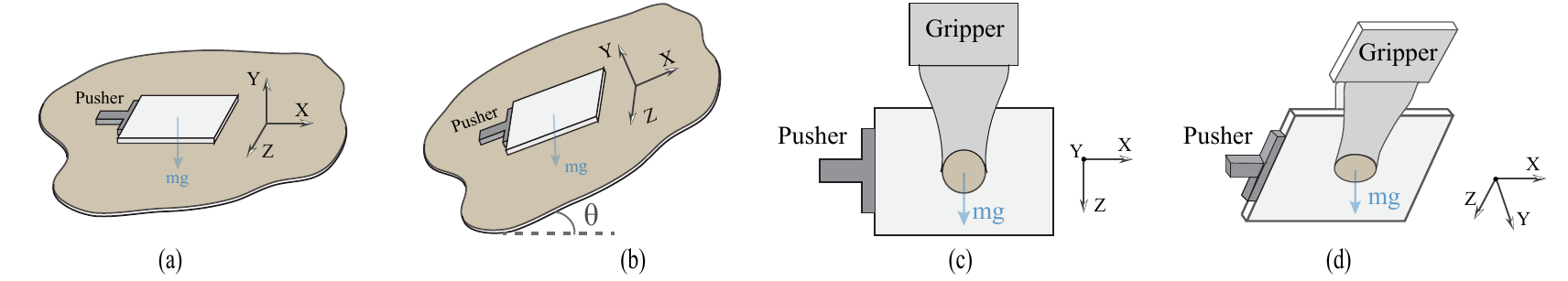}
\caption{Pushing an object (a) on a horizontal surface, (b) on an inclined surface, (c) in a grasp in the gravity plane, and (d) in a grasp in a tilted plane}
\label{fig:pushing_comparison}
\vspace{-2mm}
\end{figure*} 
\figref{fig:pushing_comparison} shows four different cases of planar pushing. In case (a), the pusher force is the only external force on the object in the plane of motion. However, in the rest of the cases, a component of gravity is also present.
The concept of the motion cone as originally studied in ~\citep{mason86} is limited to the case (a) in \figref{fig:pushing_comparison}. Our extension of the motion cone to a general planar case is valid for all the cases in \figref{fig:pushing_comparison}.

In this section we discuss the mechanics of pushing an object in a plane.
First we will review the fundamental concepts, namely, the \textit{limit surface}~\citep{GoyalPhD89}, and the \textit{generalized friction cone}~\citep{Erdmann94} that will help us model the friction interaction at the contacts involved in pushing.

\subsection{Limit Surface} 
\label{sec:limit_suface}

The limit surface is a common approach to model a friction interaction between an object and a support contact. In case (a) and (b) in \figref{fig:pushing_comparison}, the surface on which the object rests is the support contact. In
case (c) and (d), the  finger contacts play the same role.

\citet{GoyalPhD89} defined the boundary of the set of all possible friction wrenches that a contact can offer as the \textit{Limit Surface}. 
\citet{howe96, xydas99} showed that an ellipsoidal approximation allows for a simpler representation of the limit surface geometry. In this paper we will assume an ellipsoidal approximation of the limit surface, which has been shown to be computationally efficient for simulating and planning pushing motions~\cite{lynch96,Dogar2011,lynch15,Jiaji17b}.

Let $\boldsymbol{w}=[f_\textnormal{x}, f_\textnormal{z}, m_\textnormal{y}]$ be a frictional wrench on the object from the support contact in the contact frame.
A mathematical representation of the ellipsoidal limit surface is given by $\boldsymbol{w}^T A \boldsymbol{w}=1$, where $A=Diag(a_1^{-2},a_2^{-2},a_3^{-2})$.
For isotropic friction, the maximum friction force is, $a_1= a_2=\mu_{c} N$, where $\mu_{c}$ is the friction coefficient between the contact and the object and $N$ is the normal force at the contact. The maximum friction torque about the contact normal is $a_3=rc\mu_{c} N$, where $r$ is the radius of the contact and $c \in [0,1]$ an integration constant. For a uniform pressure distribution at the contact, $c$ is about 0.6~\citep{xydas99,lynch15}.

When the object slides on the support contact, the friction wrench between the object and the contact ($\boldsymbol{w}_c$)  intersects the limit surface.  Based on the maximal energy dissipation principle, the normal to the limit surface at the intersection point provides the direction of the twist of the object at the contact.
Conversely, if the object twist ($\boldsymbol{v_\textnormal{obj}}=[v_x, v_z, \omega_y]^T$) is known, we can find the friction wrench following ~\citep{lynch15} as,
\begin{equation}
\label{eq:vel2wrench}
 \boldsymbol{w}_c= \frac{A^{-1}\boldsymbol{v_\textnormal{obj\_c}}}{\sqrt{\boldsymbol{v_\textnormal{obj\_c}}^T A^{-1} \boldsymbol{v_\textnormal{obj\_c}}}}  = \mu_c N \boldsymbol{\overline{w}}_c
\end{equation}
Here, $\boldsymbol{v}_{obj\_c}$ is the velocity of the object in the contact frame and can be computed from $\boldsymbol{v_\textnormal{obj}}$ as $\boldsymbol{v}_{obj\_c}=\boldsymbol{{J}_\textnormal{c}} \cdot \boldsymbol{v}_{obj}$. The Jacobian $\boldsymbol{{J}_\textnormal{c}}$ maps the object velocity from the object frame to the support contact frame.
$\boldsymbol{\overline{w}}_c=[\overline{f}_\textnormal{x}, \overline{f}_\textnormal{z}, \overline{m}_\textnormal{y}]^T$ is the unit wrench lying on a limit surface that is scaled by maximum linear friction available at contact ($\mu_c N$) to produce the net frictional wrench. 
%

Under the ellipsoidal limit surface model assumption, translational velocity $[v_\textnormal{x\_c}, v_\textnormal{z\_c}]^T$ of the object in the contact frame is always parallel and opposite to the linear frictional force $[\overline{f}_x, \overline{f}_z]^T$ applied by the contact in the contact frame~\citep{lynch92}. Moreover, the relationship between the friction wrench and the normal to the limit surface, which defines the motion direction, sets the following constraint between the angular velocity at the contact and the linear velocity:
\begin{equation}
 \frac{v_\textnormal{x\_c}}{{\omega}_\textnormal{y\_c}}={(rc)^2} \frac{\overline{f}_\textnormal{x}}{\overline{m}_\textnormal{y}} \ \ \textnormal{
 and} \ \
 \frac{v_\textnormal{z\_c}}{{\omega}_\textnormal{y\_c}}={(rc)^2} \frac{\overline{f}_\textnormal{z}}{\overline{m}_\textnormal{y}}
\end{equation}

Given the friction wrench on the object from the support contact, we can find the object velocity as:
\begin{equation}
\label{eq:wrench2vel}
 \boldsymbol{v}_{obj} = \Tilde{k} \boldsymbol{\Tilde{J}_\textnormal{c}} \boldsymbol{B} \cdot \boldsymbol{\overline{w}}_c \ , \  \boldsymbol{B}= Diag (1, 1, (rc)^{-2} ), \ \ \Tilde{k} \in {\rm I\!R}^+
\end{equation}
Here, $\boldsymbol{\Tilde{J}_\textnormal{c}}$ maps the object velocity from the support contact frame to the object frame.

\subsection{Generalized Friction Cone}
\label{sec:generalized_cone}
The friction between the pusher and the object can be modelled with the Coulomb friction law.
\citet{Erdmann94} introduced the concept of \textit{generalized friction cone} ($\boldsymbol{W}$) as a representation of the local friction cone at a contact in the object frame. The generalized friction cone for a pusher modelled with multiple point contacts is the convex hull of the generalized friction cones for each constituent contact~\cite{Erdmann93}. 
\begin{equation}
\label{eq:wrench_cone}
 \boldsymbol{W_\textnormal{pusher}} = \{ \boldsymbol{\overline{w}_\textnormal{pusher}} = \boldsymbol{J^\top_\textnormal{p}}\cdot \boldsymbol{\overline{f}_\textnormal{p}} \ | \ \boldsymbol{\overline{f}_\textnormal{p}} \in FC_\textnormal{pusher}\}
\end{equation}

Here, $\boldsymbol{J^\top_\textnormal{p}}$ is the Jacobian that maps the local contact forces ($\boldsymbol{{f}_\textnormal{p}}$) at the pusher to the object frame. $\boldsymbol{\overline{w}_\textnormal{pusher}}$ is the unit wrench corresponding to unit force/s $\boldsymbol{\overline{f}_\textnormal{p}}$ inside the friction cone/s at the pusher constituent contact/s.

Now, with the approaches for contact modelling set, we look into formulating the mechanics of pushing in a plane.

\subsection{Mechanics of Pushing}

The motion of the object in the plane of motion evolves following the net wrench acting on it. Under the quasi-static assumption, which is appropriate for slow pushing operations, the inertial forces on the object are negligible and there is force balance:
\begin{equation}
\label{eq:force_balance}
 \boldsymbol{w_\textnormal{support}} +  \boldsymbol{w_\textnormal{pusher}} + m \boldsymbol{g}=0
\end{equation}

Here,~\eref{eq:force_balance} is written in the object frame located at the center of gravity. $\boldsymbol{w_\textnormal{support}}$ is the friction wrench provided by the support contact, $\boldsymbol{w_\textnormal{pusher}}$ is the wrench exerted by the pusher, $m$ is the mass of the object, $\boldsymbol{g}$ is the gravitational component in the plane of motion. 
The Jacobian $\boldsymbol{J^\top_\textnormal{c}}$ maps the support contact wrench from the contact frame (usually located at the center of pressure) of the support surface to the object frame. So, $\boldsymbol{w_\textnormal{support}}=\boldsymbol{J^\top_\textnormal{c}} \cdot \boldsymbol{{w}_\textnormal{c}}$ and  \eref{eq:force_balance} becomes:  
\begin{equation}
\label{eq:force_balance_prepush}
 \mu_\textnormal{c} N \boldsymbol{J^\top_\textnormal{c}} \cdot \boldsymbol{\overline{w}_\textnormal{c}} +  \boldsymbol{J^\top_\textnormal{p}}\cdot \boldsymbol{f_\textnormal{p}} + m\boldsymbol{g} = 0
\end{equation}

\subsection{Stable Pushing in a Plane}
\label{sec:stable_pushing}
%
%
\citet{lynch96} studied pushing motions for which the pusher sticks to the object when pushing on a horizontal surface, which they referred to as \textit{Stable Pushing}.
For stable pushing, force at the pusher lies inside the friction cone in the local contact frame.

For a general planar case, such a condition for a stable push can be written as:
\begin{equation*}
\begin{split}
  \mu_\textnormal{c} N \boldsymbol{J^\top_\textnormal{c}} \cdot \boldsymbol{\overline{w}_\textnormal{c}} +  \boldsymbol{J^\top_\textnormal{p}}\cdot \boldsymbol{f_\textnormal{p}} + m\boldsymbol{g} = 0 \ , \ \boldsymbol{{f}_\textnormal{p}} \in FC_\textnormal{pusher}
\end{split}
\end{equation*}
For a given object motion to be possible with a stable push, the pusher needs to be able to provide a wrench that balances the net wrench produced by the friction wrench from the support contact and the gravitational force in the plane.
Using \eref{eq:wrench_cone} we can rewrite the previous equation as:
\begin{equation}
\label{eq:stable_check_prepush1}
\begin{split}
  -\mu_\textnormal{c} N \boldsymbol{J^\top_\textnormal{c}} \cdot \boldsymbol{\overline{w}_\textnormal{c}} - m\boldsymbol{g} = k \boldsymbol{\overline{w}_\textnormal{pusher}} \\ 
  \boldsymbol{\overline{w}_\textnormal{pusher}} \in \boldsymbol{W_\textnormal{pusher}} \ , \ k \in {\rm I\!R}^+
\end{split}
\end{equation}

Here, $k$ is the magnitude of the pusher force. To know if an object motion can be achieved with a stable push, we simply need to check if the net required wrench falls inside the generalized friction cone of the pusher.
\begin{equation}
\label{eq:stable_check_prepush2}
\begin{split}
  -\mu_\textnormal{c} N \boldsymbol{J^\top_\textnormal{c}} \cdot \boldsymbol{\overline{w}_\textnormal{c}} - m\boldsymbol{g} \in \boldsymbol{W_\textnormal{pusher}}
\end{split}
\end{equation}
%


From a planning and control perspective, rather than querying if pushes are stables pushes or not, the bound on the set of object motions possible with stable pushes is more useful. This is in fact the motivation for the motion cone concept that we study in the next section.
\section{Motion Cone for Planar Pushing}
\label{sec:cone_mechanics}

The motion cone is the set of objects motions that can be produced while keeping the pusher contact sticking. In an abstract sense, it is the motion equivalent of the generalized friction cone of the pusher. 

\textbf{\textit{Problem}}: Find the set of object motions for which the net required wrench can be balanced by a wrench inside the generalized friction cone of the pusher.

This is equivalent to finding a set of object motion for which constraint \eref{eq:stable_check_prepush1} holds true. 
Rewriting,
\begin{equation*}
  -\mu_\textnormal{c} N \boldsymbol{J^\top_\textnormal{c}} \cdot \boldsymbol{\overline{w}_\textnormal{c}} - m\boldsymbol{g} = k \boldsymbol{\overline{w}_\textnormal{pusher}}, \  \boldsymbol{\overline{w}_\textnormal{pusher}} \in \boldsymbol{W_\textnormal{pusher}}, \ k \in {\rm I\!R}^+
\end{equation*}
\vspace{-4mm}
\begin{equation}
\label{eq:motioncone_eq}
\begin{split}
\boldsymbol{J^\top_\textnormal{c}} \cdot \boldsymbol{\overline{w}_\textnormal{c}} = \frac{k}{-\mu_\textnormal{c} N} \boldsymbol{\overline{w}_\textnormal{pusher}} + \frac{m}{-\mu_\textnormal{c} N}\boldsymbol{g} \\ 
  \boldsymbol{\overline{w}_\textnormal{pusher}} \in \boldsymbol{W_\textnormal{pusher}} \ , \ k \in {\rm I\!R}^+
\end{split}
\end{equation}
Using \eref{eq:wrench2vel} we can map the support contact wrench $\boldsymbol{\overline{w}_\textnormal{c}}$ to the object velocity $\boldsymbol{v}_{obj}$. Hence, to find a motion cone, we first find the set of support contact wrenches ($\boldsymbol{\overline{w}_\textnormal{c}}$) that satisfy \eref{eq:motioncone_eq} and then map this wrench-set ($\boldsymbol{\Tilde{W}_\textnormal{c}}$) to the set of object twists. We denote this object twist-set, i.e. the motion cone, by $\boldsymbol{\Tilde{V}_\textnormal{obj}}$.

The presence of an external force, (such as the gravitational force) other than the pusher force, in the plane of motion complicates the mechanics and the structure of the motion cone. To explain this effect in detail, we will first consider the case of pushing an object on a horizontal surface where there is no such additional force.

\subsection{Motion Cone for Pushing on a Horizontal Surface}
\label{sec:motioncone_horz}
\begin{figure}
\centering
\includegraphics[scale=1]{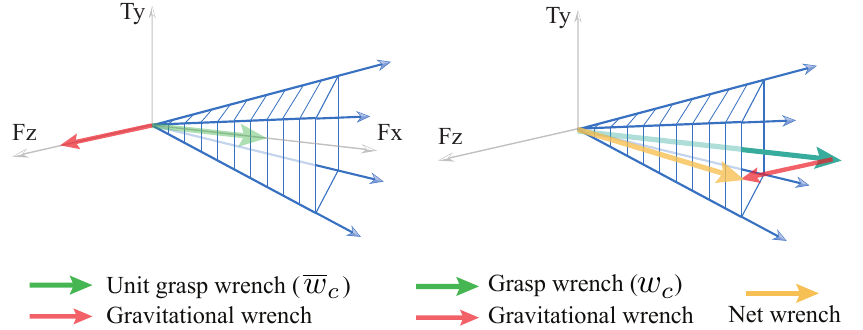}
\caption{To make a push inside the gravity-free motion cone stable in the gravity plane, the unit grasp wrench can be scaled such that the net pusher wrench required for the desired push falls inside/on the generalized friction cone of the pusher.}
\label{fig:min_scale}
\vspace{-2mm}
\end{figure} 
For the case of pushing on a horizontal plane, $\boldsymbol{g}=\boldsymbol{0}$. For an object on a flat support surface with uniform pressure distribution on the support (as is \figref{fig:pushing_comparison}- case (a) and (b)), the contact frame coincides with the object frame, $\boldsymbol{J_\textnormal{c}}= \boldsymbol{\Tilde{J}_\textnormal{c}}=\boldsymbol{I}$ (identity matrix).   
Then we can write \eref{eq:stable_check_prepush1} as:
\begin{equation*}
  -\mu_\textnormal{c} N  \boldsymbol{\overline{w}_\textnormal{c}} = k \boldsymbol{\overline{w}_\textnormal{pusher}} \ , \ \boldsymbol{\overline{w}_\textnormal{pusher}} \in \boldsymbol{W_\textnormal{pusher}} \ , \ k \in {\rm I\!R}^+
\end{equation*}
\vspace{-5mm}
\begin{equation*}
\label{eq:stablepush_horz}
 \boldsymbol{\overline{w}_\textnormal{c}}=\frac{-k \boldsymbol{\overline{w}_\textnormal{pusher}}} {\mu_{c} N} \ , \ \boldsymbol{\overline{w}_\textnormal{pusher}} \in \boldsymbol{W_\textnormal{pusher}}
 \end{equation*}

The set of valid support contact wrenches is the negative of the generalized friction cone of the pusher, i.e, $\boldsymbol{\Tilde{W}_\textnormal{c}} = -\boldsymbol{W_\textnormal{pusher}}$.
By mapping $\boldsymbol{\Tilde{W}_\textnormal{c}}$ through \eref{eq:wrench2vel}, we get the motion cone $\boldsymbol{\Tilde{V}_\textnormal{obj}}$
\footnote{\citet{mason86} defined the motion cone in terms of pusher twists which is a linear Jacobian transform of the motion cone $\boldsymbol{V_\textnormal{obj}}$. However, we will keep it in the object twist space.}.
 
Note that for the case of pushing on a horizontal surface,  $\boldsymbol{\Tilde{W}_\textnormal{c}}$ and $\boldsymbol{\Tilde{V}_\textnormal{obj}}$ are convex polyhedral cones. Moreover, they are independent of the support normal force, i.e, the weight of the object $mg$, and friction at the support surface $\mu_c$. 
    

For a more general pushing tasks however, $\boldsymbol{g}\neq\boldsymbol{0}$. 
From \eref{eq:motioncone_eq} we can see that, unlike for horizontal pushing, the system parameter ($\mu_c$) and force magnitudes ($k$ and $N$) influence the direction vectors of the wrench-set $\boldsymbol{\Tilde{W}_\textnormal{c}}$. 
In the next section we will focus on the case in \figref{fig:pushing_comparison}-(c) -- pushing an object in a parallel-jaw grasp in the plane of gravity. 
There the gravitational force is not zero in the plane of motion and the Jacobians $\boldsymbol{J_\textnormal{c}}$ and $\boldsymbol{\Tilde{J}_\textnormal{c}}$ are not always identity matrices as the support (finger) contact location changes in the object frame as the object is pushed in the grasp. 
The case shown in  \figref{fig:pushing_comparison}-(b) is same as the case (c) except the Jacobians $\boldsymbol{J_\textnormal{c}}$ and $\boldsymbol{\Tilde{J}_\textnormal{c}}$ are always identity matrices. 
The case shown in  \figref{fig:pushing_comparison}-(d) is same as the case (c) except that only a part of the gravitational force acts in the plane of motion.
%

%
\begin{figure*}[t]
\centering
\includegraphics[scale=1]{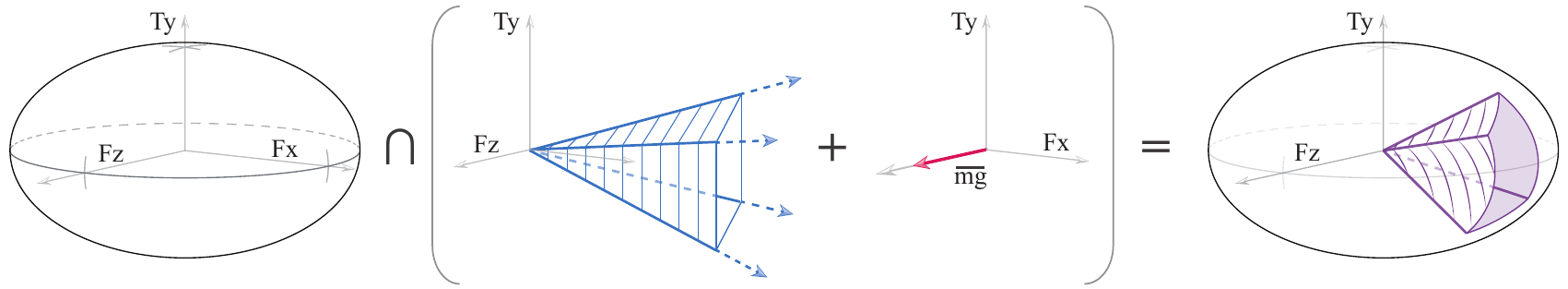}
\caption{A depiction of the process for constructing a wrench-set ($\boldsymbol{\Tilde{W}_\textnormal{c}}$).
The intersection of the limit surface with the sum of the scaled generalized friction cone of the pusher and the gravitational wrench defines the wrench-set $\boldsymbol{\Tilde{W}_\textnormal{c}}$.}
\label{fig:wrenchcone_constr}
\end{figure*}

\subsection{Stable Pushing and Motion Cone in the Gravity Plane}
For a case similar to \figref{fig:pushing_comparison}-(c), but in a gravity-free world, we can exploit the simplification of the \eref{eq:motioncone_eq} by omitting the gravity term and compute a convex polyhedral motion cone similar to that in the horizontal pushing case, but while taking the non-identity Jacobians $\boldsymbol{J_\textnormal{c}}$ and $\boldsymbol{\Tilde{J}_\textnormal{c}}$ into consideration. We will refer to this motion cone as a gravity-free motion cone in the later discussions. 

From \eref{eq:stable_check_prepush2} we see that the gravitational wrench $m\boldsymbol{g}$ on the object pulls the net required wrench in or out of the generalized friction cone of the pusher. 
Some of the motions in the gravity-free motion cone may not be stable pushes in the gravity plane, while for some object motions outside the gravity-free motion cone, the gravitational force on the object can be exploited to make them stable pushes.
%


\begin{theorem}
Any push inside the gravity-free motion cone can also be made a stable push in the gravity plane by increasing the grasping force above a minimum force threshold.
\end{theorem}

\begin{proof}
For a motion inside a gravity-free motion cone, the support/grasp wrench direction lies inside the generalized friction cone of the pusher, i.e., $\boldsymbol{J^\top_\textnormal{c}} \cdot \boldsymbol{\overline{w}_\textnormal{c}} \in \boldsymbol{W_\textnormal{pusher}}$. We can always find a magnitude ($\mu_c N$) with which the support/grasp wrench direction needs to be scaled so that the net wrench (the vector sum of the gravitation wrench and the support/grasp wrench) is inside the generalized friction cone of the pusher. 
For a given $\mu_c$, we can analytically find the bounds on $N$ needed to pull the net wrench inside the generalized friction cone of the pusher. \figref{fig:min_scale} shows the graphical interpretation.
\end{proof}

Theoretically, we can make any motion in the gravity-free motion cone a stable motion in the gravity case by changing the grasping force.
With increased grasping force, more motions in a gravity-free motion cone are stable in the gravity plane and become part of the motion cone.
For infinite grasping force, the gravity-free motion cone coincides with the motion cone in the gravity plane. However, in practice, grippers have limited grasping force. 
We need to find the motion cone for a fixed grasping force.

\begin{figure}[t]
\centering
\includegraphics[scale=0.9]{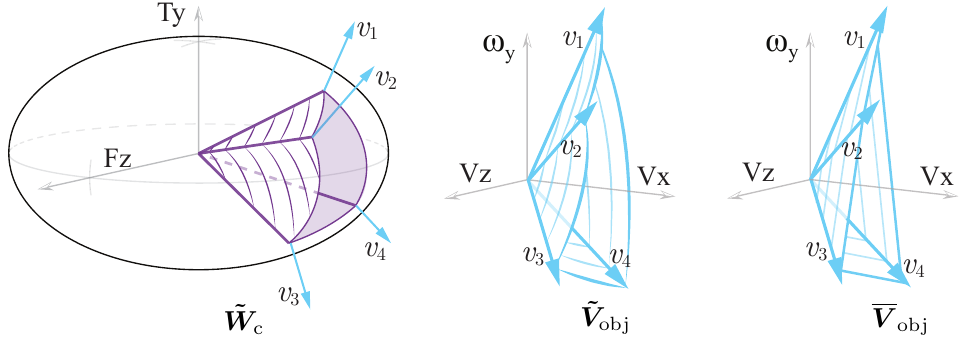}
\caption{A graphical illustration for the construction of the motion cone ($\boldsymbol{\Tilde{V}_\textnormal{obj}}$) from the wrench-set ($\boldsymbol{\Tilde{W}_\textnormal{c}}$). The motion cone is defined by the set of the surface normals to the limit surface where the wrench-set intersects the limit surface. We propose polyhedral approximation ($\boldsymbol{\overline{V}_\textnormal{obj}}$) to the motion cone for computational efficiency.}
\label{fig:motioncone_constr}
\end{figure}
\section{Motion Cone in The Gravity Plane for Fixed Grasping Force}
\label{sec:motioncone_computation}
%
%

%
In this section we find the object motion cone in the gravity case for a given grasping force and friction parameters. 

\subsection{Analytical Computation}
\label{sec:analytical_motioncone}


For a known $\boldsymbol{\overline{w}_\textnormal{pusher}} \in \boldsymbol{W_\textnormal{pusher}}$, \eref{eq:motioncone_eq} is a set of three linear equalities with 4 unknowns,  $\boldsymbol{\overline{w}_\textnormal{c}} \in {\rm I\!R}^3$ and $k$. However, we know that $\boldsymbol{\overline{w}_\textnormal{c}} =[\overline{f}_\textnormal{x}, \overline{f}_\textnormal{z}, \overline{m}_\textnormal{y}]^T$ is a unit wrench that satisfies the ellipsoidal limit surface constraint:
\begin{equation}
\label{eq:ellipsoid_wrench2}
 \frac{\overline{f}^2_\textnormal{x}}{1}+\frac{\overline{f}^2_\textnormal{z}}{1}+\frac{\overline{m}^2_\textnormal{y}}{(r c)^2}=1
\end{equation}
%
%
Constraints \eref{eq:motioncone_eq} and \eref{eq:ellipsoid_wrench2} can be solved together analytically to find $\boldsymbol{\overline{w}_\textnormal{c}}$ and $k$. Specifically, after substituting $\overline{f}_\textnormal{x}, \overline{f}_\textnormal{z},$ and $\overline{m}_\textnormal{y}$ from \eref{eq:motioncone_eq} into \eref{eq:ellipsoid_wrench2}, \eref{eq:ellipsoid_wrench2} becomes a quadratic equation in $k$. 
Solving this quadratic equation with a constraint $k \in {\rm I\!R}^+$ gives a unique solution for $k$.
Substituting this value for $k$ in \eref{eq:motioncone_eq} makes \eref{eq:motioncone_eq} a set of three linear equalities with three unknowns $[\overline{f}_\textnormal{x}, \overline{f}_\textnormal{z},\overline{m}_\textnormal{y}]$ which can be solved for a unique solution. 

For prehensile pushing in the gravity plane, the relationship between $\boldsymbol{\Tilde{W}_\textnormal{c}}$ and  $\boldsymbol{W_\textnormal{pusher}}$ is not linear as in the gravity-free case. To find wrench-set $\boldsymbol{\Tilde{W}_\textnormal{c}}$, we need to sweep $\boldsymbol{\overline{w}_\textnormal{pusher}}$ over the boundary of $\boldsymbol{W_\textnormal{pusher}}$ and solve \eref{eq:motioncone_eq} and \eref{eq:ellipsoid_wrench2} iteratively. 
\figref{fig:wrenchcone_constr} is the depiction of the process involved in solving the constraints~\eref{eq:motioncone_eq} and \eref{eq:ellipsoid_wrench2} together for computing the wrench cone $\boldsymbol{\Tilde{W}_\textnormal{c}}$. \figref{fig:motioncone_constr} shows the graphical representation of the computation in~\eref{eq:wrench2vel} that maps the wrench-set $\boldsymbol{\Tilde{W}_\textnormal{c}}$ to the motion cone $\boldsymbol{\Tilde{V}_\textnormal{obj}}$.

\begin{figure}[t]
\centering
\includegraphics[scale=0.22]{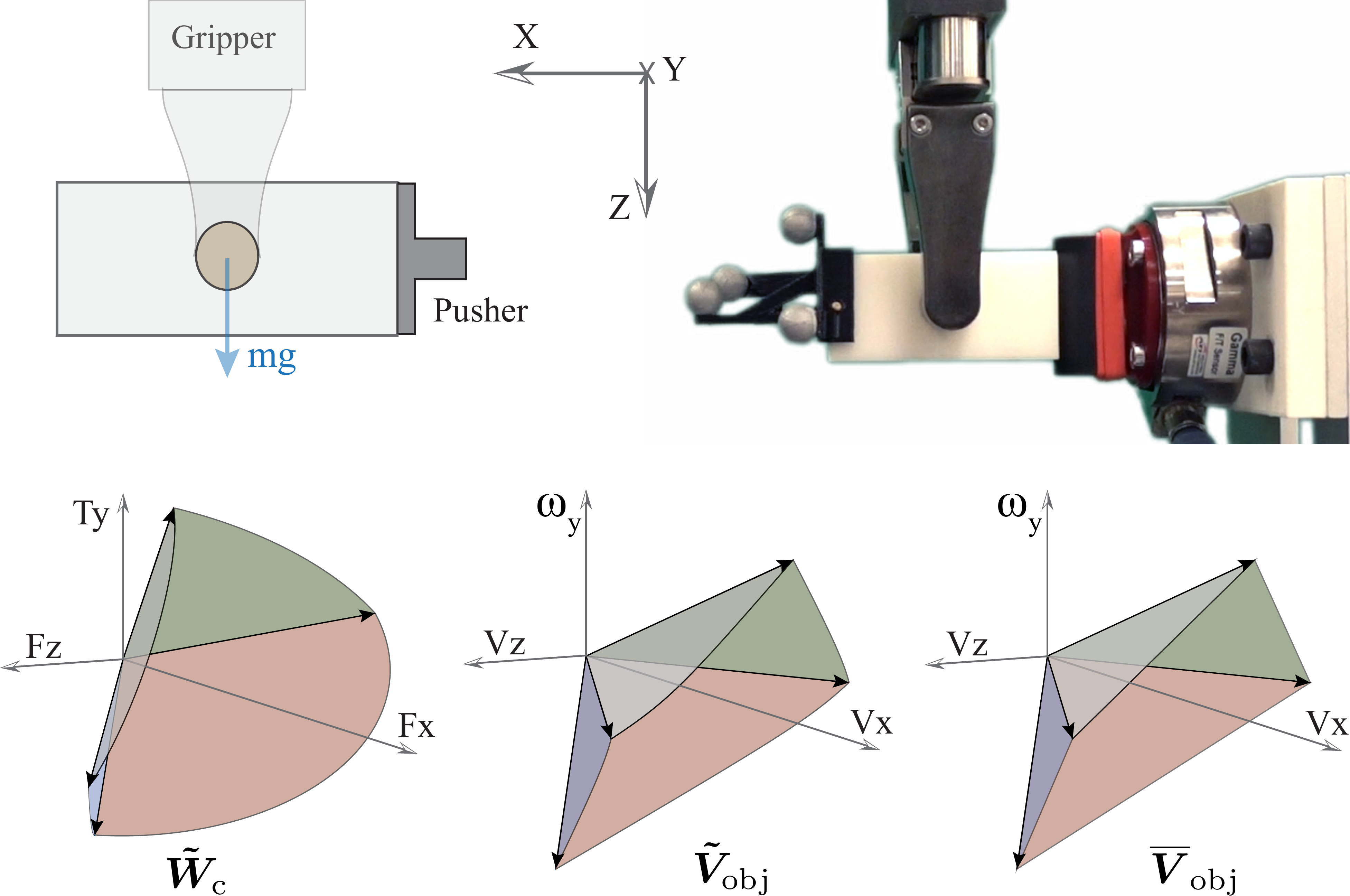}
\caption{(top) The grasp-pusher configuration used for the experimental validation of the motion cone. (bottom) The analytically computed wrench cone ($\boldsymbol{\Tilde{W}_\textnormal{c}}$), motion cone ($\boldsymbol{\Tilde{V}_\textnormal{obj}}$), and polyhedral approximation to the motion cone ($\boldsymbol{\overline{V}_\textnormal{obj}}$) for the above configuration and 45 N grasping force. Note that the surfaces defining the motion cone ($\boldsymbol{\Tilde{V}_\textnormal{obj}}$) are curved.}
\label{fig:45n_motioncone}
\vspace{-2mm}
\end{figure} 
$\boldsymbol{\Tilde{W}_\textnormal{c}}$ and $\boldsymbol{\Tilde{V}_\textnormal{obj}}$ are not polyhedral cones as in the gravity-free case, but rather can be best characterized as cones defined by low-curvature surfaces that intersect all in a point and pairwise in lines. \figref{fig:45n_motioncone} shows the wrench cone and motion cone computed analytically for one particular grasp-pusher configuration.

\subsection{Polyhedral Approximation to the Motion Cone}
\label{sec:linear_motioncone}
As an object is pushed in a grasp, the position of finger contacts in the object frame change, and consequently  $\boldsymbol{G_\textnormal{c}}$ and the motion cone $\boldsymbol{\Tilde{V}_\textnormal{obj}}$ also change. We need to compute the motion cone iteratively as the object moves in the grasp.

As the boundary surfaces of the motion cone have low curvatures, we propose a polyhedral approximation ($\boldsymbol{\overline{V}_\textnormal{obj}}$) of the motion cone for its efficient computation. Each edge of the polyhedral approximation of the motion cone is the object motion corresponding to each edge of the generalized friction cone of the pusher. \figref{fig:motioncone_constr} and \figref{fig:45n_motioncone} show the polyhedral approximation of the motion cone.

\myparagraph{Procedure to compute polyhedral motion cone}:
\begin{enumerate}
    \item Solve \eref{eq:motioncone_eq} and \eref{eq:ellipsoid_wrench2} simultaneously to get $\boldsymbol{\overline{w}_\textnormal{c}}$ for $\boldsymbol{\overline{w}_\textnormal{pusher}}$ corresponding to every edge of $\boldsymbol{W_\textnormal{pusher}}$.
    
    \item Define the set of  $\boldsymbol{\overline{w}_\textnormal{c}}$ computed in step 1 as the generators/edges of the support/grasp wrench-cone $\boldsymbol{\overline{W}_\textnormal{c}}$.
    
    \item Map $\boldsymbol{\overline{W}_\textnormal{c}}$ to the object twist space using \eref{eq:wrench2vel} to get the polyhedral approximation $\boldsymbol{\overline{V}_\textnormal{obj}}$ of the motion cone.
\end{enumerate}

\subsection{Experimental Validation of the Motion Cone}
\label{sec:exp_validation}
We use a manipulation platform equipped with an industrial robot arm, a parallel-jaw gripper, a feature in the environment that act as pusher, and a Vicon system for object tracking. 

To evaluate the experimental validity of the motion cone in the gravity plane, we collected data for the slip observed at the pusher contact for $2000$ randomly sampled pushes. 
We used the rectangular prism object listed in Table \ref{tab:objects} and the grasping force was set to $45$~N. 
\figref{fig:45n_motioncone} shows the grasp-pusher configuration used for the experiments.
Due to kinematic constraints of the robot and the workspace, we limit the sampled pushes to the space of $[0~\textnormal{mm}$-$10~\textnormal{mm}, -5~\textnormal{mm}$-$5~\textnormal{mm}, -35^{\circ}$- $35^{\circ}]$ displacements in $[X,Z,\theta_Y]$. 
%
%
\figref{fig:motioncone_exp} shows the experimental results compared to the computed one.

%
\begin{figure}
\centering
\includegraphics[scale=0.95]{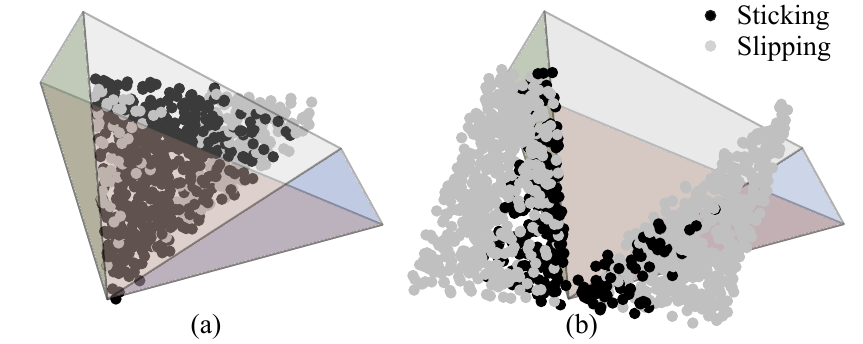}
\caption{$2000$ random prehensile pushes in the configuration shown in \figref{fig:45n_motioncone} are characterized by the slip observed at the pusher contact. The motion cone above is the same polyhedral motion cone ($\boldsymbol{\overline{V}_\textnormal{obj}}$) computed in \figref{fig:45n_motioncone}, but shown in a different orientation for better visualization.
(a) No slipping was observed for most of the pushes inside the polyhedral approximation of the motion cone, (b) Most of the pushes for which the slipping was observed are outside the motion cone.}
\label{fig:motioncone_exp}
\vspace{-2mm}
\end{figure} 
\section{Planning In-Hand Manipulations via\\ Motion Cones}
\label{sec:planning}

In this section we demonstrate the application of motion cones for planning in-hand manipulations with stable prehensile pushes.
The problem formulation is similar to the one presented in our recent work \citep{ChavanDafle2018a}. 
We assume that an object is grasped in a parallel-jaw gripper and manipulated in the gravity plane by pushing against features in the environment.
We assume the following physical properties of the system:
\begin{itemize}
    \item[$\Cdot$] Object geometry and mass.
    \item[$\Cdot$] Initial and goal pose of an object in a grasp, specified by the locations of each finger contacts.
    \item[$\Cdot$] Gripping force.
    \item[$\Cdot$] Set of pusher contacts, specified by their geometries and locations in the object frame.
    \item[$\Cdot$] Coefficient of friction at all contacts.
\end{itemize} 

In~\citep{ChavanDafle2018a}, we present a planning framework where at the high-level, a T-RRT$^*$-based architecture samples different object poses in the grasp~\citep{trrt_star,trrt}. At the low level, a rejection check is implemented using a constraint similar to \eref{eq:stable_check_prepush2} to evaluate if the sampled configuration can be reached using a stable prehensile push.
We adopt the same high-level T-RRT$^*$-based planning approach in this paper, but without the low-level rejection check. Rather, the planner always grows the tree towards the sampled pose as best as possible using the motion cones at the nearest node to the sampled pose.


\myparagraph{For selective exploration}, the TRRT* framework relies on a transition test that filters the sampled configurations to prefer exploration in low configuration-cost regions. We define the configuration cost as the distance from the goal. The transition test softly constrains the stochastic exploration towards the goal grasp, while allowing the flexibility to explore high-cost transitions if they are necessary to get the object to the goal.

\myparagraph{For effective connections}, the T-RRT$^*$ algorithm uses the underlying RRT$^*$~\citep{rrt_star} framework to make and rewire the connections in the tree at every step such that the cost of the nodes is reduced when possible.
We define the cost of a node as the sum of the cost of the parent node and the cost of the push to reach the sampled node from the parent node. We set the cost of a push $0.1$ if the parent node uses the same pusher as the child and $1$ otherwise.  
%
With our node cost definition, the planner generates pushing strategies that prefer fewer pusher switch-overs to push the object to the desired pose.

\begin{algorithm}
  \caption{: In-Hand Manipulation Planner}\label{alg:full_planner}
  $  \textbf{input}: q_{init}, q_{goal}$ \par
  $  \textbf{output}:$ {tree} $\ \mathcal{T}$
  \begin{algorithmic}
  \State $\mathcal{T}\gets \textrm{initialize tree}(q_{init})$
  \State $ \textrm{generate\_motionCones}(\mathcal{T},q_{init})$
      \While{$q_{goal} \notin \mathcal{T}$ \textbf{or} cost($q_{goal}) > \textrm{cost threshold} $}
        \State $q_{rand}\gets \textrm{sample random configuration}(\mathcal{C})$
        
        \State $q_{parent}\gets \textrm{find nearest neighbor}(\mathcal{T},q_{rand})$
        \State $q_{sample}\gets \textrm{take unit step}(q_{parent},q_{rand})$
        \If{$q_{sample} \notin \mathcal{T}$}
        
        \If{\textrm{transition test}$(q_{parent},q_{sample},\mathcal{T})$}
            \State $q_{new} \gets \textrm{motionCone\_push}(q_{parent},q_{sample})$
            
            \If{\textrm{transition test}$(q_{parent},q_{new},\mathcal{T})$ \textbf{and} \\ \hspace{19mm}\textrm{grasp maintained}$(q_{new})$} 

                \State $q\mbox{*}_{parent}\gets \textrm{optimEdge}(\mathcal{T},q_{new},q_{parent})$
                
                \State $\textrm{add new node}(\mathcal{T},q_{new})$
                
                \State $\textrm{add new edge}(q\mbox{*}_{parent},q_{new})$
                
                \State $ \textrm{generate\_motionCones}(\mathcal{T},q_{new})$
                
                \State $\textrm{rewire tree}(\mathcal{T},q_{new},q\mbox{*}_{parent})$
            \EndIf
            \EndIf
        \EndIf
    \EndWhile
  \end{algorithmic}
\end{algorithm}

Let $q$ denote a configuration of an object, i.e., the pose of the object in the gripper frame, which is fixed in the world.
In this paper, we are considering planar manipulations in a parallel-jaw grasp, so the configuration space $\mathcal{C}$ is $[X, Z, \theta_y] \in {\rm I\!R}^3$, i.e., the object can translate in the grasp plane ($XZ$) and rotate about a perpendicular ($Y$) to the grasp plane.

Algorithm \ref{alg:full_planner} shows our in-hand manipulation planner.
Let $q_{init}$ and $q_{goal}$ be an initial and desired pose of the object in the gripper frame respectively. 
The planner initiates a tree $\mathcal{T}$ with $q_{init}$ and generates motion cones at $q_{init}$. 

While the desired object pose is not reached within some cost threshold, a random configuration $q_{rand}$ is sampled. A nearest configuration $q_{parent}$ to $q_{rand}$ in the tree $\mathcal{T}$ is found and an unit-step object pose $q_{sample}$ towards $q_{rand}$ is computed. 
Using the transition test, the planner evaluates if moving in the direction of $q_{sample}$ from $q_{parent}$ is beneficial or not.
If it is beneficial, the \textit{motionCone\_push} routine computes an object configuration $q_{new}$ closest to $q_{sample}$ that can be reached using the motion cones at $q_{parent}$.
It is further checked if moving towards $q_{new}$ is beneficial and if $q_{new}$ is an object configuration at which the grasp on the object is maintained. 
If both the criteria are satisfied, $q_{new}$ is added to the tree such that the local cost of $q_{new}$ and the nodes near $q_{new}$ are lowered if possible. The motion cones are generated for every new node added to the tree.

Two important routines in Algorithm \ref{alg:full_planner}, particularly for this paper, are \textit{generate\_motionCones} and \textit{motionCone\_push}.

\myparagraph{\textit{generate\_motionCones}} computes polyhedral motion cones for a given object configuration in the grasp using the procedure listed in \secref{sec:linear_motioncone}. At every node, we will have the same number of motion cones as that of pushers.

\myparagraph{\textit{motionCone\_push}} finds an object pose closest to the desired sampled pose $q_{sample}$ that can be reached. This computation is done using the motion cones at the parent node $q_{parent}$.
If the object twist needed from the parent node pose to the sampled pose is already inside any of the motion cones, the sampled pose can be directly reached.
If the required object twist is outside all the motion cones, a twist that is inside one of the motion cones and closest to the desired twist is selected.
%

The use of motion cones for fast low-level unit-step propagation of the system and T-RRT$^*$-based framework for high-level planning allows us to explore the configuration space of different object poses in the grasp and generate pushing strategies for the desired in-hand manipulation.

\begin{table}[b]
\vspace{-3mm}
  \caption{Physical properties of the experimental objects}
  \label{tab:objects}
	\centering
	\begin{tabular}{|l|l|c|r|}
         \hline
          \textbf{Shape} & \textbf{Material} & \textbf{Dim [L, B, H]} (mm) & \textbf{Mass} (g) \\ \hline
          square prism  & Al 6061 & 100, 25, 25 & 202\\ \hline  
          rectangular prism  & Delrin & 80, 25, 38 & 113\\ \hline
          T-shaped & ABS  & 70, 25, 50 & 62\\ \hline
    \end{tabular}
\end{table}


\begin{table}
\caption{Planning times (sec.) for approaches using motion cone, stable check~\cite{ChavanDafle2018a} and MNCP~\cite{ChavanDafle2017} for unit-step propagation}
  \label{tab:timing}
	\centering
	\begin{tabular}{|l|c|r|r|r|}
         \hline
          \textbf{Manipulation} & \textbf{Goal} & \textbf{Planning}  & \textbf{Planning}   & \textbf{Planning}\\
           & [$X, Z, \theta_Y $] & \textbf{Time} & \textbf{Time}~\cite{ChavanDafle2018a} & \textbf{Time}~\cite{ChavanDafle2017} \\
          \hline
          Horz. offset (low $\mu$)  & 20, 0, 0 & 0.45 & 2.83 & 592.8\\ \hline
          [$X, Z, \theta_Y$] Regrasp & 15, -13, 45 & 0.67 & 2.54 & 17684\\ \hline
          T-shaped & 25, 17.5, 0 & 0.54 & 0.82 & 32657\\ \hline
    \end{tabular}
\end{table}
    
\section{Regrasp Examples and Experimental Results}
\label{sec:examples}
We evaluate the performance of our planner with examples of a parallel-jaw gripper manipulating a variety of objects. 
%
%
The initial pose of an object in the gripper is treated as $[X, Z, \theta_Y]=[0,0,0]$. 
\tabref{tab:timing} lists the goal poses (in [mm, mm, deg.]) for different examples. 
While there are no comparable available algorithms that can solve the type of regrasps we are interested in, we provide comparisons with our own implementations of the same high-level planner paired with different algorithms to solve the mechanics of prehensile pushing. These include sampling with rejection by a feasibility check for stable pushing~\cite{ChavanDafle2018a}, and a complementarity
formulation (MNCP) that allows both sticking and slipping at the pusher contact~\cite{ChavanDafle2017}. We compare the performance in terms of planning time and the quality of the solutions.
The planning times in \tabref{tab:timing} are the median times over 10 trials.
%
All the computations are done in MATLAB R2017a on a computer with Intel Core i7 2.8 GHz processor.

In all the examples below, we assume three line pushers on the object, one on each side faces of the object parallel to the $Z$ axis and one under the object  parallel to the $X$ axis. We use high friction line pushers, except in the first example.

\begin{figure}
\centering
\includegraphics[scale=.14]{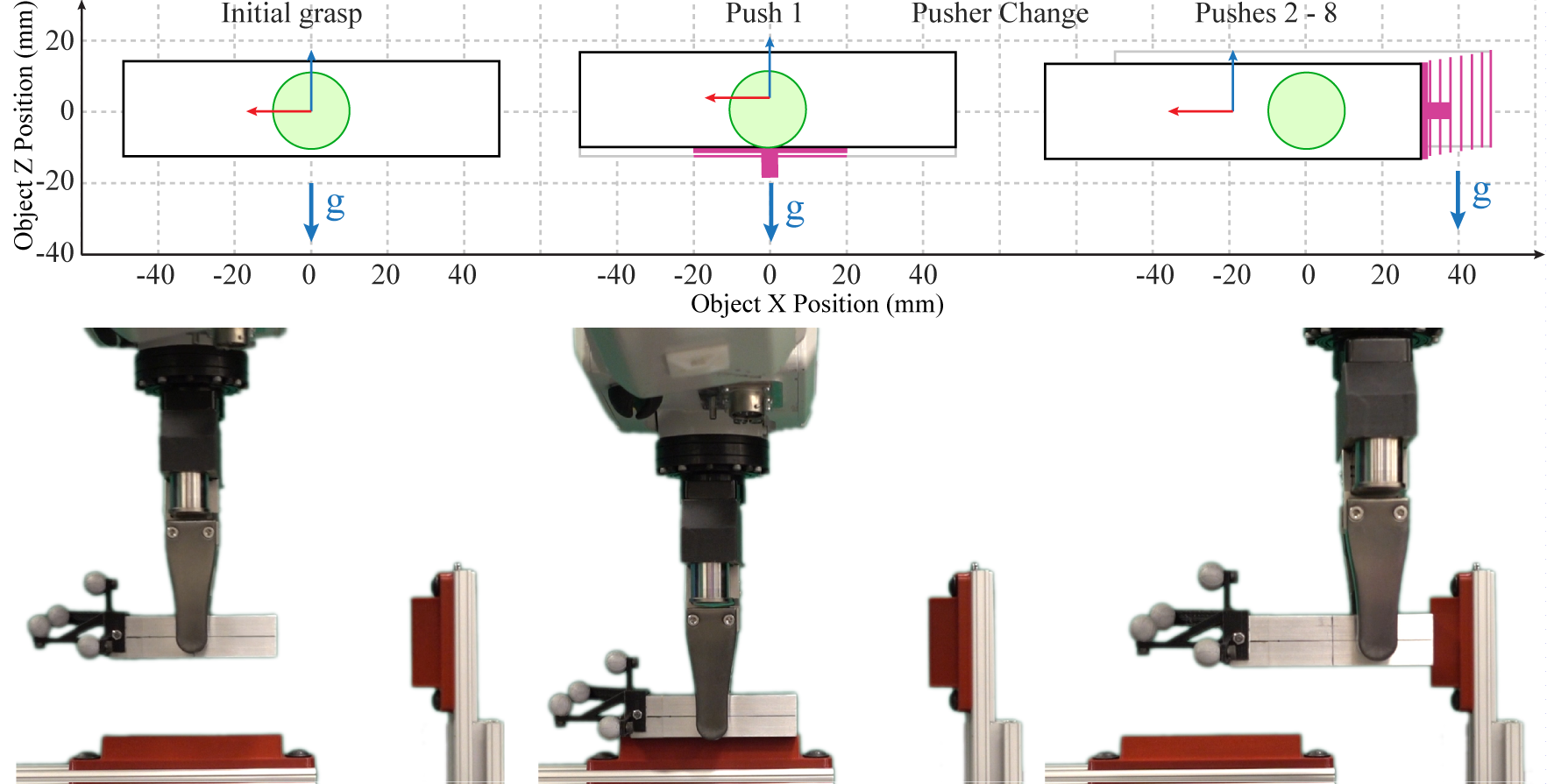}
\caption{Simulation and experimental run for a pushing strategy to regrasp the aluminum object with low friction pushers. In the simulation figure (top), the finger and pusher contacts are shown in green and magenta color respectively.}
\label{fig:linpush_lowfric}
\end{figure} 
\begin{figure}
\centering
\includegraphics[scale=0.155]{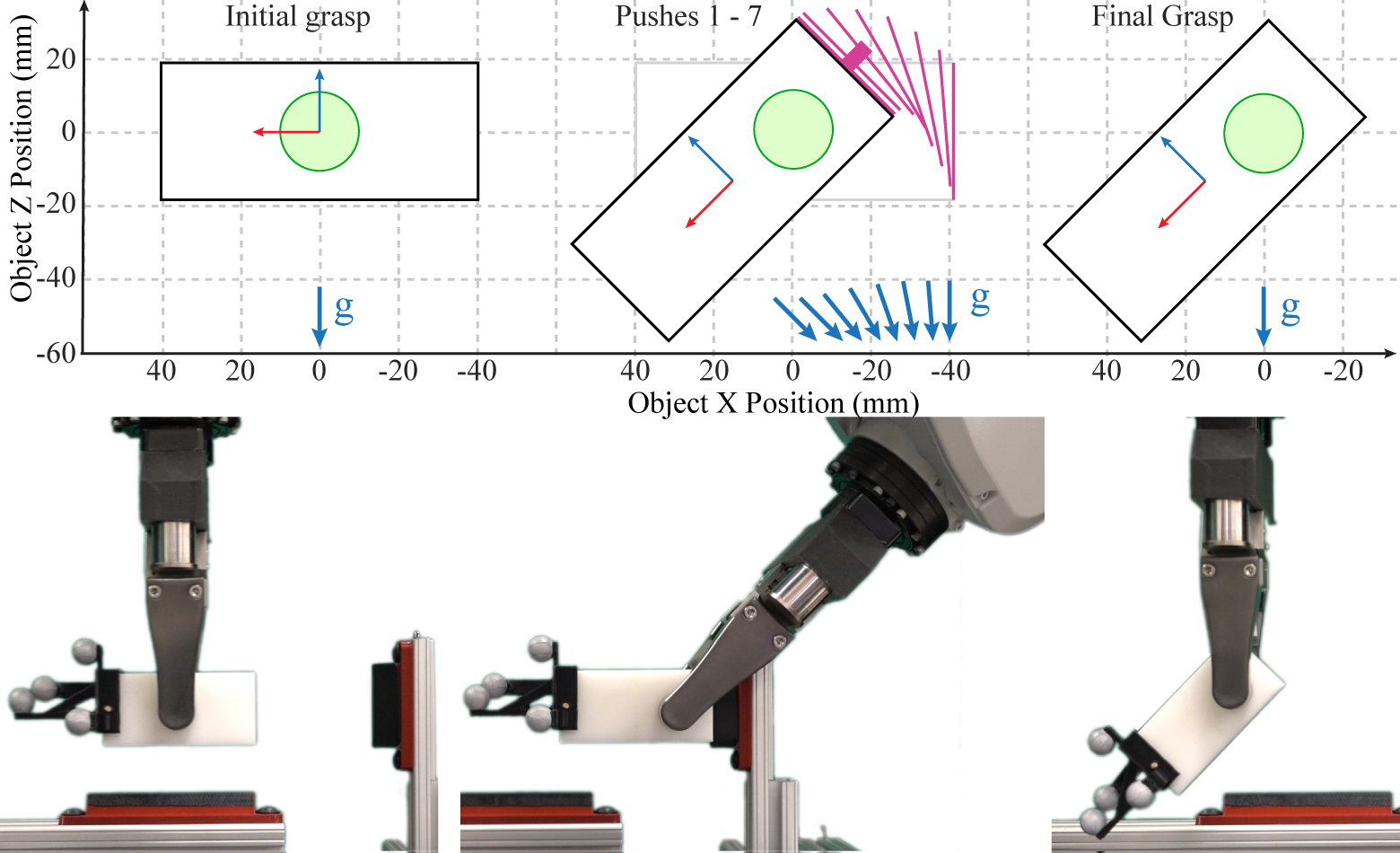}
\caption{A pushing strategy for [$X, Z, \theta_Y$] regrasp. In simulation, the direction of gravity remains constant in the pusher frame, because in reality, the pushers are fixed features.}
\label{fig:prepush}
\vspace{-3mm}
\end{figure}
\subsubsection{Regrasping an object offset to the center}
In this example, the goal is to regrasp the square prism horizontally 20 mm offset from the center. We use low friction pushers first. \citet{kolbert16} showed that for a similar setting, if the object is pushed horizontally in the grasp it slides down as it moves sideways in the grasp.  
For the low-friction pushers, our planner generates a strategy where the object is first pushed up using the bottom pusher and then the side pusher is used to virtually keep the object stationary while the fingers slide up and along the length of the object as seen in~\figref{fig:linpush_lowfric}. This plan is similar to the one found in~\citep{ChavanDafle2017, ChavanDafle2018a}.

When we replace the pushers with high-friction pushers (pushers with rubber coating), the planner detects that the desired object twist lies inside the motion cone for the side pusher at the initial grasp pose, i.e, simply pushing from the side is a valid pushing strategy.

\subsubsection{Regrasp in [$X, Z, \theta_Y$]}
The goal in this example is to regrasp the rectangular prism requiring twist in all three dimensions [$X, Z, \theta_Y$]. Similar to \citep{ChavanDafle2018a}, our planner finds a strategy to achieve the regrasp using only one pusher. In fact, as we can see in \figref{fig:prepush}, the pushing strategy our planner comes up with is more direct and seems to avoid unnecessary object motions seen in the strategy shown in \citep{ChavanDafle2018a}-Fig.1. 


\subsubsection{Manipulating a non-convex object}
In this example, the goal is to regrasp a T-shaped object. The goal pose is such that a greedy approach to push the object directly towards the goal will result in losing the grasp on the object.
\begin{figure}
\centering
\includegraphics[scale=0.88]{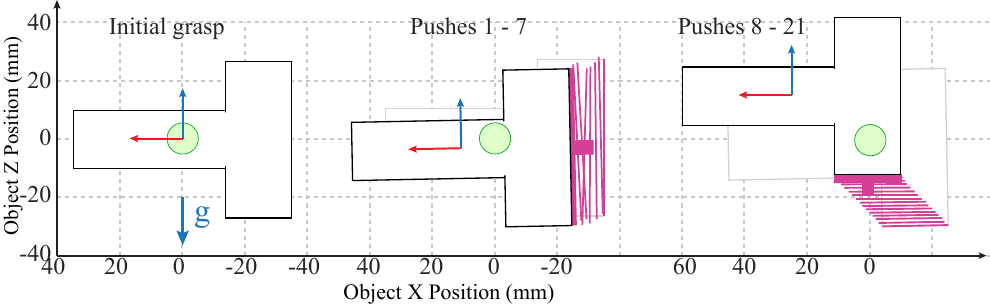}
\caption{Simulated motion of the object in the grasp for a pushing strategy to manipulate T-shaped object. Snapshots of the experimental run are is shown in \figref{fig:tpush_exp}}
\label{fig:tpush_sim}
\vspace{-3mm}
\end{figure}
Our planner comes up with  a pushing strategy that respects the geometric constraints of the problem as shown in  \figref{fig:tpush_sim}.


\section{Discussion}
\label{sec:discussion}

The motion cone is the set of possible twists of a  pushed object. 
It also describes the set of motions of the pusher that yield sticking behavior.
It abstracts away the complex dynamics of frictional pushing and provides direct bounds in the action or effect space for pushing tasks. 

%
In this paper we extend the concept of motion cones to a general set of planar pushing tasks with external forces such as the gravitational force in the plane of motion. 
We show that the motion cone for a general planar push is defined as a cone with low-curvature faces, and propose a polyhedral approximation for efficient computation. 

We demonstrate the use of motion cones as the propagation step in a sampling-based planner for in-hand manipulation. Combining a T-RRT$^*$-based high level planning framework and a \emph{motion cone}-based dynamics propagation, the planner builds in-hand manipulation strategies with sequences of continuous prehensile pushes in a fraction of a second. 

The motion cone provides a direct knowledge of the set of reachable configurations. Such a structure of reachable volumes could enable planning through regions/volumes of configuration space~\cite{brock2001decomposition,morales2007analysis,Russ09,Russ11}. Moreover, motion cones, as bounds on pusher actions or as bounds on the object motions, have an adequate form to be incorporated into trajectory optimization frameworks to plan pushing strategies. We believe that the extension and application of motion cones to more general settings provides new opportunities for fast and robust manipulation through contact.

\section*{Acknowledgments}
This material is based upon the work supported by Mathworks, the Chang Award and the MIT Merrill Lynch Fellowship. We would like to thank the members of the Mcube Lab, particularly Francois Hogan, Nima Fazeli, and Maria Bauza for helpful discussion and advice. Special thanks to Prof. Matt Mason for suggestions on the graphical illustrations of the motion cone in \figref{fig:wrenchcone_constr} and \figref{fig:motioncone_constr}. 


\bibliographystyle{plainnat}
\bibliography{ncd-rss18}

\end{document}